\documentclass[runningheads]{llncs}
\usepackage[T1]{fontenc}
\usepackage{graphicx}

\usepackage{amsmath}
\usepackage{amssymb}
\usepackage{hyperref}
\usepackage{algorithm}
\usepackage{algpseudocode}
\usepackage{amsfonts}
\usepackage{multirow}
\usepackage{booktabs}
\usepackage{xcolor}
\usepackage{color}
\usepackage{url}
\usepackage{ulem}
\usepackage{booktabs}
\usepackage{adjustbox}
\usepackage{graphicx}
\usepackage{subcaption}

\newcommand{\yp}[1]{\textcolor{orange}{Yiwen: #1}}

\newcommand{\ignore}[1]{}
\newcommand{\added}[1]{\textcolor{black}{#1}}

\renewcommand{\vec}[1]{\overrightarrow{#1}}
\renewcommand{\emph}[1]{\textit{#1}}
\begin{document}
%
%\title{Fuzzy Logic-based Reasoning for Unsupervised Knowledge Graph Alignment}
\title{FLORA: Unsupervised Knowledge Graph Alignment by Fuzzy Logic} % Fabian: Here is a proposal: Fuzzy Logic implies reasoning, so no need to mention it. By saying "with FLORA", we make sure the acronym appears in the title, in case people search for it...
% FLORA: Fuzzy LOgic-based Reasoning for Unsupervised Knowledge Graph Alignment
\author{Yiwen Peng\orcidID{0009-0007-7902-4097} \and Thomas Bonald\orcidID{0000-0003-0468-0384} \and Fabian M. Suchanek\orcidID{0000-0001-7189-2796}}
% \author{Anonymous}
% \and
% Second Author\inst{1}\orcidID{1111-2222-3333-4444} \and
% Third Author\inst{1}\orcidID{2222--3333-4444-5555}}
% %
\authorrunning{Y. Peng et al.}

\institute{Télécom Paris, Institut Polytechnique de Paris, Palaiseau, France 
\\
\email{\{yiwen.peng, thomas.bonald, fabian.suchanek\}@telecom-paris.fr}
\\
}
%\url{http://www.springer.com/gp/computer-science/lncs} 
% \and
\maketitle              % typeset the header of the contribution
\begin{abstract}
Knowledge graph alignment is the task of matching equivalent entities (that is, instances and classes) and relations across two knowledge graphs. Most existing methods focus on pure entity-level alignment, computing the similarity of entities in some  embedding space. They lack interpretable reasoning and need training data to work.
In this paper, we propose FLORA, a simple yet effective method that (1) is unsupervised, i.e., does not require training data, (2) provides a holistic alignment for entities and relations iteratively, % Yiwen: I want to show the keywords "iterative" in abstract.
%(3) integrates semantics and structural signals, % Fabian: that is how we do it, but it is not an advantage. If we had done it by any other method, that would have been also good.
(3) is based on fuzzy logic and thus delivers interpretable results, (4) provably converges, (5) allows dangling entities, i.e., entities without a counterpart in the other KG, and (6) achieves state-of-the-art results on major benchmarks.
%Indeed, 
% Fabian: I was a bit radical here, feel free to come back to the original text or merge the two.
% old versions:
%sources, which is important to knowledge fusion and integration. 
%Most existing methods focus on pure instance-level alignment, computing the similarity of entities via an embedding-learning-and-matching paradigm. They lack interpretable reasoning to prevent logically wrong mappings and heavily rely on ideal seed alignments for good performance.
%Inspired by the probabilistic reasoning system PARIS, we introduce a simple yet effective framework, FLORA, that retains the reasoning capability while integrating semantic and structural patterns into a unified rule for KG alignment. Specifically, we initialize the similarity scores of literal contexts of entities using pretrained language models to bootstrap the algorithm, and then iteratively update entity and relations alignment via fuzzy logic until convergence. We further provide a theoretical proof of convergence, which is absent in PARIS system.
%
% Fabian: in the interest of space I propose to cut:
% Yiwen: tranfer to contributions in introduction.
%Indeed, extensive experiments demonstrate that FLORA consistently outperforms baselines on 5 entity alignment datasets covering 10 distinct KGs across four languages, as well as two KG alignment datasets from the OAEI KG Track, achieving notable improvements in alignment precision and recall scores.

\end{abstract}
%
%
%

% Introduction
\section{Introduction}

The task of Knowledge Graph (KG) alignment consists in matching both entities and relations in one KG to their equivalents in another. 
Figure~\ref{fig:example} shows a toy example of  alignment between DBpedia and Wikidata.
 % \fms{if there is any chance to show a (toy) example from YAGO, that would help the visibility of our team... Also, maybe consider using Marco Polo instead of the miniseries ``Marco Polo''. The person Marco Polo was a traveler who united China and Europe in friendship -- a nice message to pass nowadays...}.
% \yp{Dbpedia-yago usually has same names which is trivial for matching, that's why i don't include it. I try to find examples that 1)different entity names. 2) has subrelations matching. 3) relations not that functional but joint functional (4) has matched inverse relations, (5) has meaning literals, (6) in our experimental datasets. Marco Polo don't satisfy these:(.}
% Fabian: I am trying to follow the Stanford 5 here: (0) what is the problem (1) why is it important (2) why is it hard (3) why has nobody else done it (4) how do we do it
% https://cs.stanford.edu/people/widom/paper-writing.html
% Fabian: Again, don't hesitate to modify!
% (1) why is it important 
KG alignment is useful for knowledge fusion~\cite{xin2022large,wang2022facing} and thus for a wide range of  downstream applications like question answering~\cite{dong2023hierarchy,abi2023psychic}, common-sense reasoning~\cite{liu2021kg}, and recommender systems~\cite{chen2024macro,huang2023disentangled}.
% KG Alignment plays an important role in  \fms{add citations}. 
% (2) why is it hard
It is a challenging task for several reasons.  First, KGs are generally  heterogeneous and incomplete: One KG may contain information that the other one does not contain. Second, entity names can differ  vastly across KGs (in Figure~\ref{fig:example} for instance, entities in Wikidata are represented by non-readable IDs, unlike in DBpedia). Third, training data are most often  not available in practice, due to the high cost of manual annotation. Finally, the alignment of entities and relations  is interdependent.
% (3) why has nobody else done it

\begin{figure}[h]
\centering
\includegraphics[width=\textwidth]{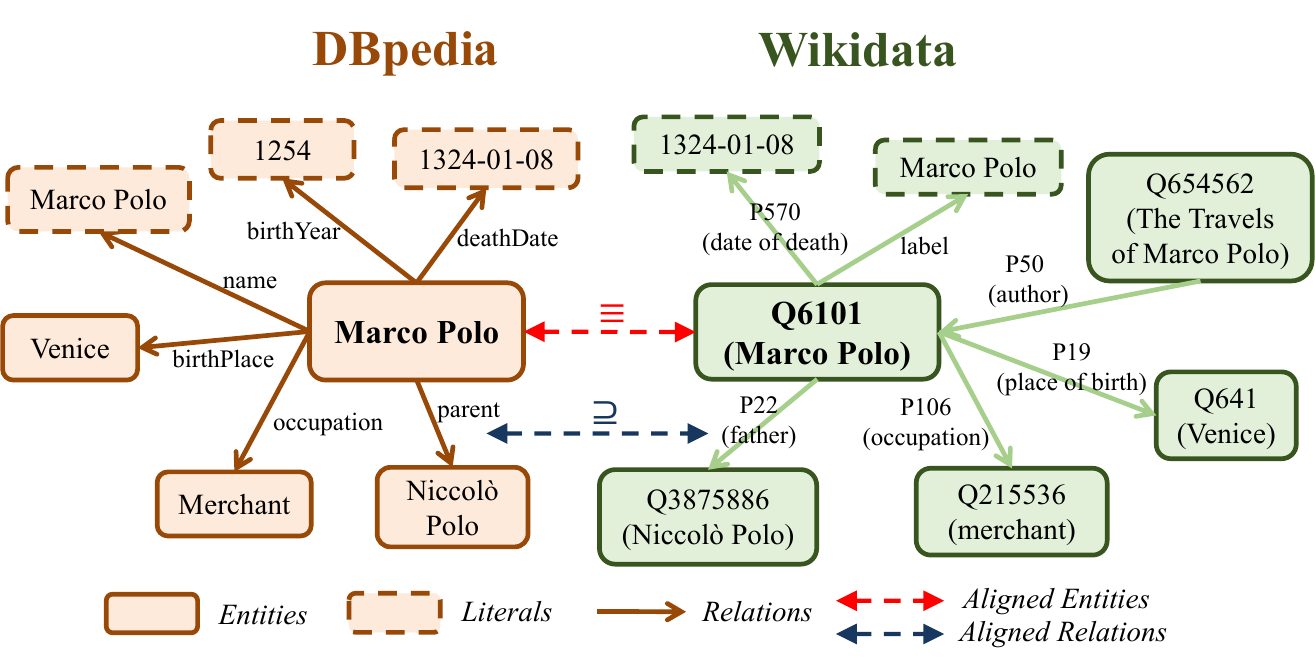}
\caption{An excerpt of DBpedia and Wikidata KGs. Note the asymmetric mapping of the relation \emph{parent} to the sub-relation \emph{father}.} \label{fig:example}
\end{figure}

Many approaches exist for KG alignment, but most  focus on  the alignment of entities only. %; very few  apply to  relation alignment as well~\cite{multike,rhgn,imea}.
% \fms{the ones you cite \emph{can} do rel alignment, right?}. % Yiwen: Yes. they also say soft relation alignment.
Moreover, they generally   require training data \cite{2020experimental}
or assume a one-to-one entity mapping from one KG to the other (i.e, ignore dangling entities)~\cite{fgwea}.
%, which is unrealistic as many KGs are constructed independently and vary in scale.
% of one another and have different scales.
% \fms{find a neat characterization of related work: is there none that fulfills all desiderata at the same time?}
% \yp{LogMap or SEU, but still they are not that interpretable or no better results.}. 
A notable exception is PARIS~\cite{paris}, an approach based on the iterative computation of probabilistic equations that  has been shown to outperform other methods in recent studies~\cite{paris_plus,openea}. However, this approach (1) lacks convergence guarantees, (2) does not achieve good performance when functional relations are absent, and (3) cannot see similarities between literals beyond a  strict identity.

% Yiwen: Have not finished yet from here
% Fabian: OK! Also note the page limitation...
% Yiwen: Finished
% (4) how do we do it
To address these shortcomings, we present FLORA\footnote{Fuzzy-Logic based Object and Relation Alignment},
%\fms{Just for the record: LUISA -- Logic-based unsupervised instance and schema alignment; AURELIA -- Automated unsupervised Relation and Instance Alignment; and generally PARIS 2}, 
a novel unsupervised approach for aligning both entities and relations. FLORA formalizes the alignment in Fuzzy Logic, in a principled framework that integrates various signals, both semantic and structural,  and yields human-interpretable explanations for the results. %Semantic signals are captured by initializing soft literal similarities using pretrained language models, while structural signals are formalized by the functionality of joint relations. 
%FLORA   provides a soft relation alignment between KGs. 
FLORA works even in the presence of dangling entities and  in the absence of functional relations. 
%In Figure~\ref{fig:example} for instance, \textit{musicalBand} in DBpedia is not strictly equivalent to \textit{performer(P175)} in Wikidata, as \textit{performer} can also represent \textit{dancer, actor, singer}\footnote{\url{https://www.wikidata.org/wiki/Property:P175}}. So we align \textit{musicalBand} as a subrelation of \textit{performer(P175)} in our framework.
% \yp{Asymetric relations and dangling entities in example.}
It relies on an iterative algorithm that provably converges. 
Our contributions are the following:
\begin{itemize}
    \item[$\bullet$] FLORA, an  approach based on Fuzzy Logic for aligning  entities and relations that accounts for the asymmetry of relation alignments and the incompleteness of KGs.  Although unsupervised by nature, FLORA can optionally exploit training data if available.
    \item[$\bullet$] An iterative algorithm that provably converges towards the optimal solution of the  problem, as formalized in Fuzzy Logic.
    %integrating both semantic and structural signals. The algorithm is supported by a solid theoretical foundation, under which we formally prove its convergence. 
    \item[$\bullet$] Extensive experiments on five entity alignment datasets across four languages, as well as two KG alignment datasets from the OAEI KG Track, showing that FLORA consistently outperforms both unsupervised and supervised baselines.
    %in terms of F1 and Hit@1 scores.
    %Indeed, extensive experiments demonstrate that FLORA consistently outperforms baselines on 5 entity alignment datasets covering 10 distinct KGs across four languages, as well as two KG alignment datasets from the OAEI KG Track, achieving notable improvements in alignment precision and recall scores.
\end{itemize}
%This paper is structured as follows: 
\noindent In what follows, Section~\ref{sec:relwork} reviews related work, Section~\ref{sec:preliminaries} provides preliminaries, Section~\ref{sec:fuzzy-logic} introduces our Fuzzy Logic formalism, Section~\ref{sec:method} presents our method, Section~\ref{sec:experiments} shows experimental results, and Section~\ref{sec:conclusion} concludes and outlines future work. \added{All our code and data are publicly available at \url{https://github.com/dig-team/FLORA}.}
\section{Related Work}\label{sec:relwork}
% Fabian: this is to be drastrically shortened, I made a backup in OLD

% \fms{For each approach, say (1) what it does and (2) why we are different or better. For example: The XYZ system also aligns entities across KGs. Different from FLORA, XYZ cannot... / does not achieve SOTA performance on... / ...}

\paragraph{\textbf{Entity Alignment.}}% According to the survey~\cite{2020experimental}, Most EA studies rely on supervised framework using translation-based embeddings [JAPE, AttrE, BootEA, TransEdge] or GNN encoders [GCNAlign, RDGCN, AttrGNN], which heavily depend on seed notations that require massive manual work and may not be available in practice.
% Recently, unsupervised EA methods have addressed the lack of human-labeled pairs by leveraging side information within KGs.
Many existing approaches to entity alignment (EA) work in a supervised setting \cite{2020experimental}.
They use either translation-based embedding~\cite{jape,bootea,attre,fuzzyea,transedge},  GNN encoders~\cite{gcnalign,rdgcn,attrgnn,imea,rhgn} or fine-tuned language models~\cite{bert-int,tea,sdea,chatea}. 
%\fms{do we compare to these in our experiments?}\yp{Yes. All the papers in related work are the baselines in the experiments.} % Fabian: fantastic!
Different from these approaches, FLORA does not require training data.
In the unsupervised setting, non-neural methods such as LogMap~\cite{logmap}, PARIS~\cite{paris}, PRASE~\cite{prase}, and NALA~\cite{nala} exploit logical reasoning and lexical matching. Other methods transform the EA task into an optimal assignment problem~\cite{cpl-ot,seu,fgwea}, or leverage names, attributes and relations to improve the  alignment of entities \cite{multike}.
 Neural methods, such as ICLEA~\cite{iclea}, SelfKG~\cite{selfkg}, LLM4EA~\cite{llm4ea} and HLMEA~\cite{hlmea},  adopt a self-supervised setting or  leverage large language models (LLMs) in a zero-shot manner. Unlike FLORA,  these methods are \emph{black-box} methods, i.e., they cannot explain why a given alignment was found. 
%FLORA, in contrast, can show the human-readable rules that led to the alignment of two entities.
%These methods are often resource-intensive and suffer from limited interpretability due to the black-box nature of their framework. In contrast, our method achieves competitive performance through a simple, interpretable design.
% \fms{do we compare to these methods?} Yiwen: Yes.
Besides, we show in our experiments that FLORA outperforms all of the above approaches on standard benchmark datasets. Moreover, FLORA aligns not just entities, but also relations at the same time.

%\vspace{-0.5em}
\paragraph{\textbf{Knowledge Graph Alignment.}}
% As a more challenging task, Knowledge Graph Alignment needs holistic solutions that address a more general problem for aligning both instances (like entities) and schema (classes and relations). 
% Plenty of approaches can either match instances (like EA)[VeeAlign~\cite{veealign}, SOBERT~\cite{sobert}], or align schema (classes and relations)[BERTMap~\cite{bertmap}, OLaLa~\cite{olala}], but quite a few can do holistic matching for both. The existing methods for holistic matching are mostly included in OAEI KG track,  we include some competitive ones.
The task of KG alignment (also known as KG matching or holistic KG alignment) extends EA by aligning not only instances, but also classes and relations. \added{Unlike Ontology Alignment~\cite{euzenat2007ontology}}, KG Alignment takes a simplistic view on classes, and typically does not consider asymmetric class subsumptions across KGs.  
% \yp{To discuss if we say this.}\fms{better?}
%While this task overlaps with ontology alignment, \yp{it do not consider the ontological commitment inherent in taxonomy structures, like the transitive closure of subclass and the directed acyclic graph structure of taxonomy.}
The Ontology Alignment Evaluation Initiative (OAEI)\footnote{\url{https://oaei.ontologymatching.org/}} organizes a KG track tailored for KG alignment. Methods such as OLaLa~\cite{olala} and SOBERT~\cite{sobert} align only classes. Our work, in contrast, aims at performing a holistic matching of instances, classes, and relations. Some methods like PARIS~\cite{paris}, LogMap~\cite{logmap}, ATMatcher~\cite{atbox}, and AgreementMakerLight~\cite{aml} are based on lexical matching. Others, such as Wiktionary~\cite{wiktionary} or ALOD2Vec~\cite{alod2vec} incorporate external background knowledge. We can show that FLORA achieves better results than these methods without requiring additional external datasets.

\paragraph{\textbf{Fuzzy Logic.}} Very few approaches use Fuzzy Logic for KG alignment. FuzzyEA~\cite{fuzzyea} relies on a standard embedding approach for entity alignment and Fuzzy Logic is used only to filter candidate entity pairs. This is very different from our approach, where Fuzzy Logic is central in the alignment process, for both entities and relations. \added{In other works~\cite{hnatkowska2022fuzzy,kozierkiewicz2023fuzzy,todorov2014fuzzy,fernandez2012fuzzyalign}}, Fuzzy Logic is used for ontology alignment, but is limited to either class matching or instance matching.

% Yiwen: I've checked that the papers proposed by reviewers are all about "concepts (classes) alignment using fuzzy logic". In early years, kgs = ontologies + instances. So ontology matching is just for classes and relations. 

% Preliminaries
 \section{Preliminaries}\label{sec:preliminaries}
A knowledge graph $\mathcal{G}$ is based on a set of entities $\mathcal{E}$ and a set of relations $\mathcal{R}$. In this paper, in line with~\cite{chen2023rethinking,wang2022facing,chen2020learning}, we see a KG 
% \fms{prevent Semantic Web nitpickers from challenging this definition.} 
as a set of triples (or facts) $\mathcal{T}\subset \mathcal{E}\times \mathcal{R} \times \mathcal{E}$. 
%In what follows, we assume a fixed set $\mathcal{T}$, which we omit if it is clear from the context. 
Each triple $(h, r, t)\in \mathcal{T}$ represents a semantic relationship between the head entity $h$ and the tail entity $t$ with the relation $r$. We denote by $r(h,t)$ the corresponding boolean proposition, which evaluates to \emph{true} if and only if $(h, r, t)\in \mathcal{T}$. Without loss of generality, we assume that each relation $r\in \mathcal{R}$ has an inverse relation $r^{-1}\in \mathcal{R}$, such that $r(h,t)$ if and only if $r^{-1}(t,h)$. 
In this paper, the entity set $\mathcal{E}$ includes literals, instances, and classes. 

\paragraph{\textbf{Functionality.}}
%A relation $r$ is a \emph{subrelation} of another relation $r'$, written $r\subseteq r'$, if for any entities $h,t$, $r(h,t) $ implies $r'(h,t)$. 
% We further assume that there is a relation $\textit{type} \in \mathcal{R}$. 
Like PARIS \cite{paris}, our approach exploits the notion of functionality of a relation, i.e., the degree to which a relation tends to map the head entity to a unique tail entity. 
%Take the relation \textit{BornIn} as an example, if the relation functionality is equal to one, it means that, given a person, its birthplace can be uniquely determined. 
Formally, the functionality of a relation $r\in \mathcal{R}$ is:
\[
\textit{fun}(r) = \frac{\left| \{ h \mid \exists t: \ r(h,t) \} \right|}
{\left| \{ (h, t) \mid r(h,t) \} \right|}
\]
Note that $\textit{fun}(r)\in [0,1]$ with $\textit{fun}(r)=1$ if and only if the relation $r$ is functional: It maps any head entity to a unique tail entity. 
For example, the relation \textit{hasCapital} generally assigns exactly one capital city to each country, so that  $\textit{fun}(\textit{hasCapital}) \approx 1$. If the entities referring to \textit{France} in each KG are aligned, the corresponding entities referring to \textit{Paris} can then be aligned as well.

Our approach also uses the notion of  \textit{local} functionality, which is specific to the head entity $h$:
\[
\textit{fun}(r,h) = \frac{1}
{\left| \{ t \mid r(h,t) \} \right|}
\]
Again, $\textit{fun}(r,h)\in [0,1]$ with $\textit{fun}(r,h)=1$ if and only if there is a unique entity $t$ such that $r(h,t)$ is true. The local functionality is needed to deal with relations that are globally functional, but have local exceptions.
 For instance, \textit{South Africa} has three capitals, \textit{Pretoria}, \textit{Bloemfontein}, and \textit{Cape Town}, so that the relation \textit{hasCapital} is not locally functional for the head entity \textit{South Africa}. %Enforcing local functionality is key to robust entity alignment. % Fabian: too early to make such a statement

\paragraph{\textbf{KG Alignment.}} Given two knowledge graphs $\mathcal{G} =\langle\mathcal{E},\mathcal{R}, \mathcal{T}\rangle$, $\mathcal{G'}  =\langle\mathcal{E'},\mathcal{R'}, \mathcal{T'}\rangle$, entity alignment (EA) is the task of finding the set $$\mathcal{M}_e = \{(e, e') \mid e \equiv e', e \in \mathcal{E}, e' \in \mathcal{E}'\}$$ where $e \equiv e'$ means that entity $e$ is semantically equivalent to $e'$. 
% Fabian: It seems quite natural that relations have to be aligned asymmetrically...
%We also consider the \textit{soft} alignment of relations like~\cite{multike} to handle relation asymmetry where \textit{soft} means that we allow subrelations rather than strict equivalence.
Relation alignment is the task of finding the set  $$\mathcal{M}_{r} = \{(r, O, r') \mid r O  r', r \in \mathcal{R}, r' \in \mathcal{R}', O \in \{\subseteq, \supseteq, \equiv\}\},$$
where $\subseteq, \supseteq, \equiv$ refer to subrelation, superrelation, or equivalence, respectively. 
%We refer to this task as {\it soft} relation alignment.
%\fms{If we do experiments on classes, we have to define it here:} Analogously, class alignment is the task of identifying $\mathcal{M}_c = \{(c, O, c') \mid c O c' \wedge c \in \mathcal{C} \wedge c' \in \mathcal{C}' \wedge O \in \{\subseteq, \supseteq, \equiv\}\}$.
%, where $c \equiv c'$ means that $c$ is semantically equivalent to $c'$ (asymmetric subclass relationships are typically not considered \fms{PARIS does subclass matches actually. So if we define subclass matches away here, reviewers will complain that there are approaches that do subclass matches... Should we define class matching in the same way as relation matching, with $O$?}). 
KG alignment is the task of doing the  alignment of both entities and relations at the same time.
 \section{Fuzzy Logic Framework} \label{sec:fuzzy-logic}

% We apply fuzzy logic  to implement logical operators in a learning-free manner. 
% The matching is mainly split into two steps: 1) entity matching and 2) subrelations mapping. The detailed logical formulas are described below.

% Our alignment method, called FLORA, is based on Fuzzy Logic. This allows us to integrate signals of different nature,  semantic and structural, in a single principled framework. 

Our method, FLORA, is based on Fuzzy Logic. This allows us to integrate signals of different nature,  semantic and structural, in a single principled framework. Our framework can use arbitrary aggregation functions, as long as they are continuous and non-decreasing. This distinguishes our framework from Probabilistic Soft Logic~\cite{bach2017hinge}, which commonly requires linear aggregation functions (usually Łukasiewicz T-norms, \added{or G\"odel T-norms for negatively weighted rules~\cite{dickens2021negative}), 
and from standard propositional G\"odel logic, which uses \emph{min} aggregation~\cite{baaz2007first}}.

% rather than being restricted to the standard \emph{min} conjunction in infinite-valued propositional G\"odel logic.
% Yiwen: Different from Probabilistic Soft Logic [1], which requires linear aggregation functions (typically Łukasiewicz T-norms, or Gödel T-norms for negatively weighted rules [2]) to guarantee the convexity of the maximum-a-posteriori inference problem, our framework only requires continuous and non-decreasing functions (e.g., \textit{min}, harmonic mean, or arithmetic mean).

% \fms{I reformulated, please check.}}

\paragraph{\textbf{Fuzzy Inference System.}} A \emph{fuzzy set} $A$ over a universe of discourse $X$ is given by a membership function $\mu_A$, which maps every element of $X$ to a value in $[0,1]$. In an often-used example, $X$ is the range of temperatures in degrees Celsius and $A$ is the fuzzy set that describes ``hot temperatures''. The membership function of $A$ is then, for example, $\mu_A(x)=0$ for $x\leq 10^{\circ}$, $\mu_A(x)=(x-10^{\circ})/15^{\circ}$ for $x$ between $10^{\circ}$ and $25^{\circ}$, and $\mu_A(x)=1$ for $x\geq 25^{\circ}$. In its simplest form, a Mamdani-style Fuzzy Rule Based Reasoning System~\cite{fis}, or Fuzzy Inference System (FIS), consists of a set of rules of the form
\[p_1~ \textit{is } P_1 \wedge ... \wedge p_n ~\textit{is } P_n  \Rightarrow  c~ \textit{is } C \]
\noindent Here, $p_1,...,p_n$ are input variables (we call them \emph{premises}), $P_1,...,P_n$ are fuzzy sets, $c$ is the output variable (which we call the \emph{conclusion}), and $C$ is a fuzzy set. For example, $p_1$ can be the temperature of the room in degrees Celcius, $P_1$ is the fuzzy set that describes ``hot temperature'',  $p_2$ is the number of people in the room, $P_2$ is the fuzzy set of ``a crowded room'', $c$ is the desired speed of the fan (say, in rotations per minute), and $C$ is the fuzzy set of ``a fast rotation''. The expression ``$p_i~ \textit{is}~ P_i$'' (for $i=1,...,n$) evaluates to $\mu_{P_i}(p_i)$ (in a process called \emph{fuzzification}). 
In our example, a room temperature of $p_1=20^{\circ}$ belongs to the fuzzy set $P_1$ of hot temperatures to a degree of $\frac{2}{3}$, and a number of people $p_2=10$ could belong to ``a crowded room'' $P_2$ to a degree of 0.5. 

The \emph{firing strength} of the rule is then computed by aggregating the individual fuzzified values of each premise using a custom aggregation function $\phi$. Usually, $\phi$ is a T-norm \cite{fodor2004left}, but we will not impose this restriction in this paper. A common choice for $\phi$ is \emph{min}, which yields 0.5 in our example. The output of the rule is a fuzzy set for $c$, which is given by the membership function of $C$, capped at the firing strength of the rule. 
In our example, $\mu_C$ could assign a value of 1 to a speed of more than 60 rotations per minute, and would  be capped to 0.5. If several rules have the same output variable ($c$ in our example), their fuzzy sets ($C$ in our example) are combined, usually using a point-wise max operation on the membership functions. This yields one final fuzzy set for the output variable, which is then \emph{defuzzified} into a single scalar value $c^*$.

\paragraph{\textbf{Simple Positive FIS.}} For our work, we need a simpler form of FIS, where all  membership functions are identity functions, there is no negation, and defuzzification is done by the First of Maxima (FoM) method, i.e., the aggregated fuzzy set $C$ is defuzzified into the scalar value $c^*=\min \{ c | \forall x, \mu_C(c)  \ge \mu_C(x) \}$. Our example rule can thus be written more simply as:
\[ \textit{temperature} \wedge \textit{crowded} \Rightarrow \textit{fanspeed} \]
\noindent Here, \textit{temperature} is a normalized value for the temperature in $[0,1]$, \textit{crowded} is a normalized value for the degree of crowdedness in $[0,1]$, and \textit{fanspeed} is the normalized speed of the fan in $[0,1]$. If we use \emph{min} as aggregation, this rule means that \emph{fanspeed} must be identical to the minimum of \emph{temperature} and \emph{crowdedness}. If several rules have \emph{fanspeed} as conclusion, then it will be the maximum of the firing strengths of the rules. More formally:

\begin{definition}[Simple Positive FIS]
A Simple Positive FIS $F$ is a set of rules of the form $P_1 \wedge ... \wedge P_n \stackrel{\phi}{\Longrightarrow} C$, where $P_1,...,P_n$ are input variables (premises) that have a given value in $[0,1]$, $C$ (the conclusion) is an output variable whose value is to be computed, and $\phi:\mathbb{R}^n \rightarrow [0,1]$ is an aggregation function. A rule of $F$ is \emph{satisfied} if the value of the conclusion is greater than or equal to the rule strength $\phi(P_1,...,P_n)$. A \emph{solution} of $F$ is an assignment of each output variable to a value so that (1) all rules are satisfied, and (2) no output variable can be assigned a smaller value and still satisfy all rules.
\end{definition}

\paragraph{\textbf{Recursive FIS.}} We now extend the Simple Positive FIS by allowing the conclusions to appear as premises of rules:
\begin{definition}[Recursive FIS]\label{def:recfis}
A recursive FIS is a Simple Positive FIS where the output variables can appear as premises.
\end{definition}
In our example, a fast-running fan may attract more people seeking chilling air:
\[ \textit{fanspeed} \Rightarrow \textit{crowded} \]
The notions of rule satisfaction and solution carry over to recursive FISs.

%\fms{we somehow have to justify why we use Fuzzy Logic... competitors are PSL, Markov Logic, Probabilistic Logic, etc...}\tb{Do we?}
% Fuzzy logic differs from propositional logic in that every logical formula has a soft truth value in $[0, 1]$. Formally, a valuation function for a given set $\mathcal{P}$ of propositional variables is a function $v:\mathcal{A} \rightarrow [0,1] $. The valuation function is extended as follows to formulas of propositional logic: $v(\neg \phi)=1-v(\phi)$, $v(\psi \vee \phi)=v(\neg(\neg\psi \wedge \neg\phi))$, $v(\psi \wedge \phi)=T(v(\psi),v(\phi))$, and $v(\phi\Rightarrow\psi)=sup \{ z | t(\} $. Here, $t(\cdots)$ is t-norm, i.e., a mapping $t: [0,1]^2 \to [0,1]$ satisfying the properties of commutativity, associativity, monotonicity and neutrality of 1.  $c($
% A usual  $t$-norm is the \textit{minimum}, a standard operator for conjunction, used in G\"odel fuzzy logic. \fms{define co-norm} Correspondingly, disjunction is represented by the \textit{maximum}. \fms{define negation and implication}
%In this paper, we generalize the operators of conjunction and disjunction to arbitrary functions that are 
%use different operators for the conjunctions, including $t$-norm and its relaxation. All operators are 
%commutative and monotonic. We write the operator as a superscript, as in $\stackrel{\text{min}}{\Longrightarrow}$ for \textit{minimum}.

\paragraph{\textbf{Fixed point iteration.}}
Algorithm~\ref{alg:fis} computes the solution of a Recursive FIS. It starts by assigning all output variables the value of zero, computes the firing strength of each rule, and updates the value of the output variable  of this rule. This process is iterated until convergence.
%
%\vspace{-0.15cm}
\begin{algorithm}
\caption{Solve Recursive FIS}\label{alg:fis}
\begin{algorithmic}
\Require A recursive FIS $F$
\Ensure The solution $v$ of $F$
\State $v(x) \gets 0$ for all output variables $x$ of $F$
\While {no convergence}
\For {rule $r$ of $F$}
  \State $\textit{strength}(r) \gets $ firing strength of $r$ with values $v$
  \State $x \gets$ output variable of $r$
 \State $v(x) \gets \max(v(x),\textit{strength}(r))$
 \EndFor  
\EndWhile
\State return $v$
\end{algorithmic}
\end{algorithm}
%We describe the fuzzy alignment as some vector $x\in [0,1]^K$ where $$K = |\mathcal{E}| \times |\mathcal{E'}| + 2 \times|\mathcal{R}| \times |\mathcal{R'}|.$$ 
%Each component of $x$ corresponds to the value of some entity alignment $e\equiv e'$ for $e\in \mathcal{E}$, $e'\in \mathcal{E}'$ or some subrelation alignment $r\subset r'$ or $r'\subset r$ for $r\in \mathcal{R}$, $r'\in \mathcal{R}'$.
% We have: % \yp{Notations} % Fabian: I fixed the notations
%\vspace{-0.2cm}
\begin{theorem} \label{theo:conv}
If each aggregation function  is continuous and non-decreasing, then
Algorithm~\ref{alg:fis} converges to the solution of the input recursive FIS. %, and there is no solution that is component-wise smaller. % Fabian: this is already part of the definition of ``solution''
\end{theorem}
\begin{proof}
Let $K$ be the number of output variables. Let us see the assignment $v$ of Algorithm~\ref{alg:fis} as a vector $\vec{v}\in [0,1]^K$, with one value per output variable.  Each iteration of the algorithm corresponds to the update $\vec{v}\gets f(\vec{v})$ for some mapping 
$f:[0,1]^K\to [0,1]^K$.
A solution to the FIS must satisfy $\vec{v}\ge f(\vec{v})$ component-wise so that all rules are satisfied. Since $\vec{v}\le f(\vec{v})$ component-wise, this means that $\vec{v}=f(\vec{v})$, i.e., $\vec{v}$ is a fixed point of $f$. Now since $f$ is non-decreasing, because so are the aggregation functions  involved in $f$, as well as the $\textit{max}$ operator,  it follows from the Knaster-Tarski fixed point theorem \cite{tarski} that the set of fixed points is a lattice and that there is a unique least fixed point. It remains to prove that Algorithm \ref{alg:fis} converges to this least fixed point.

The algorithm computes a sequence of vectors $\vec{v_0}=0,\vec{v_1},\vec{v_2},\ldots$ by applying $f$ at each iteration.
Since $\vec{v}\le f(\vec{v})$ component-wise for any vector $\vec{v}$,
we have $\vec{v_0} \le \vec{v_1}$ component-wise. Since $f$ is non-decreasing, we deduce that $f(\vec{v_0}) \le f(\vec{v_1})$ component-wise, that is $\vec{v_1} \le \vec{v_2}$. By induction, we obtain that  the sequence $\vec{v_0}, \vec{v_1}, \vec{v_2},\ldots$ is non-decreasing. Since this sequence is upper bounded by the vector of ones, it has some limit $\vec{v}^\star$ in $[0,1]^K$, showing that the algorithm converges. 

To prove that $\vec{v}^\star$ is a fixed point of $f$, we use the continuity of $f$, which again follows from the fact that the aggregation functions involved in $f$ and the $\textit{max}$ operator are continuous.
We have:
$$
f(\vec{v}^\star) = f(\lim_{t\to \infty} \vec{v_t})
= \lim_{t\to \infty} f(\vec{v_t})
= \lim_{t\to \infty} \vec{v_{t+1}}
= \vec{v}^\star.
$$
Finally, $\vec{v}^\star$ is 
  the least fixed point of $f$ because for any other fixed point of $f$, say $\vec{v}$, we have $\vec{v_0}=0 \le \vec{v}$ and thus $\vec{v_1} = f(\vec{v_0}) \le f(\vec{v}) = \vec{v}$ component-wise.  By induction, we get $\vec{v_{t+1}} = f(\vec{v_{t}}) \le f(\vec{v}) = \vec{v}$ component-wise for all $t\ge 0$. Taking the limit  gives $\vec{v}^\star\le \vec{v}$ component-wise. 
\hfill$\Box$
\end{proof}

 Our result extends to the scenario where some or all output variables have an initial value that shall not be undercut. It suffices to add a rule of the form $\textit{c} \Rightarrow x$ for each output variable $x$ with initial value $c$.
 This allows for an alternative definition of a recursive FIS: A recursive FIS is a set of rules of the form $P_1 \wedge ... \wedge P_n \stackrel{\phi}{\Longrightarrow} C$, where $\phi$ is an aggregation function, and $P_1,...,P_n$ and $C$ are variables, each of which has an initial value in $[0,1]$. A solution is then an assignment of each variable to a value so that (1) all variables have values that are larger than or equal to their initial value, (2) all rules are satisfied, and (3) no variable can be assigned a smaller value and still satisfy all rules. Such an alternative FIS can be translated to a FIS according to Definition~\ref{def:recfis} by considering all variables that appear in conclusions as output variables, considering the others as input variables, and adding a rule $c\Rightarrow x$ for any output variable $x$ with initial value $c$.

 % The result extends to the case where the initial assignment v is different from 0. It is enough to add a rule of the form value ⇒ variable for each variable whose initial value is not zero
% ----------->  Fabian: not correct! If the initial assignment is not zero, then the optimal solution of the problem might not be found! What we mean is: when the problem is different, we can still treat it.

% Implementation

\section{FLORA} \label{sec:method}

We now apply our framework of Fuzzy Logic to the task of KG alignment. We first extend the notion of functionality to relation lists, which is exploited in FLORA to align  entities that are objects of  multiple facts.

%For this purpose, we first relax the notion of functionality: while a single relation may not map the head entity to a unique tail entity, the combination of several relations may identify a unique tail entity given the head entities.

\subsection{Relation lists}

For an alphabetically sorted list $R=(r_1,\ldots,r_n)$ of relations, a list of head-entities $H=(h_1,\ldots,h_n)$, and a tail entity $t$, we define $R(H,t)$ as the boolean proposition that all corresponding facts exist:
$$R(H,t) := r_1(h_1,t)\land \ldots \land r_n(h_n,t).
$$
By convention, the head entities corresponding to the same relation are   sorted in alphabetic order. This ensures that there is a unique way to write $R(H,t)$ even if several relations in $R$ are identical. We can then define the functionality of the relation list $R$ as: 
$$
\textit{fun}(R) = \frac{\left| \{ H \mid \exists t,\ R(H, t)\} \right|}
{\left| \{ (H, t) \mid R(H,t)  \} \right|}
$$
For instance, the relations {\it BirthDateOf} and {\it FamilyNameOf} are themselves not functional, but their combination is, as few people have both the same birth date and family name.
Similarly, the local functionality of relation list $R$ and entity list $H$ is defined as:
\[
\textit{fun}(R, H) = \frac{1}
{\left| \{ t \mid R(H, t) \} \right|}
\]
We have $\textit{fun}(R, H)=1$ if and only if there is a unique tail entity $t$ such that $R(H,t)$ is {\it true}. Again, this notion is used to cope with relation lists that are globally functional, but have local exceptions.
% \fms{maybe examples here} 
% Yiwen: Added!
%The functionality of a relation list $R$ captures both local and global graph structure. As shown in Figure~\ref{fig:example}, each single relation alone is not functional enough to determine whether $\texttt{Kaze waFuiteiru} \equiv \texttt{Q134085}$. However, when considering the joint relations $R = (\texttt{recordLabel}^{-1}, \texttt{musicalBand}^{-1}, \texttt{releaseDate}^{-1}, \texttt{subsequentWork})$, the overall functionality reaches 1.
%This indicates that there is a unique entity following the same structural pattern, enabling FLORA to accurately match entities even in the absence of strongly functional single relations.

%The global functionality is the  harmonic mean of the local functionalities $F(R,H)$, over all list of entities $H$ such that $R(H,t)$ holds.

%\subsection{Matching Rules}

% \begin{figure}
% \includegraphics[width=1.1\textwidth]{figures/illustration.pdf}
% \caption{An illustration of how the entity rules work in real graphs. The graphs are excerpts of DBPedia and Wikidata from the D-W-15K-V2 dataset.} \label{fig3}
% \end{figure}

\subsection{ Alignment Rules}\label{sec:matchingrules}

In order to align two KGs $\mathcal{G}$ and $\mathcal{G'}$, we construct a recursive Fuzzy Inference System. The key insight, similar to PARIS \cite{paris}, is the following: If there exist facts $r(h,t)$ and $r'(h',t')$ such that
% (for some $r \in \mathcal{R}, h \in \mathcal{E}, t \in \mathcal{E}, r' \in \mathcal{R'}, h' \in \mathcal{E'}, t' \in \mathcal{E'}$), 
entities  $h, h'$ have already been matched and  $r,r'$ are functional, then entities $t,t'$ are matched. 
%Compared to PARIS,   FLORA (1) accounts for local functionality, (2) exploits relation lists, and (3) is based on Fuzzy Logic instead of a probabilistic setting, making the algorithm provably convergent (thanks to Theorem \ref{theo:conv}) and the results explainable.  
FLORA goes beyond PARIS by generalizing this insight to lists of relations, by embedding it in the framework of Fuzzy Logic, and by proving its convergence.

%\paragraph{\textbf{Entity alignment.}}'
\paragraph{\textbf{Entity alignment}.} 
% We are given two knowledge graphs $\mathcal{G}$ (with entities $\mathcal{E}$, relations $\mathcal{R}$ and facts $\mathcal{T}$) and $\mathcal{G'}$ (with entities $\mathcal{E'}$, relations $\mathcal{R'}$ and facts $\mathcal{T'}$) to align. 
 For each pair of entities $t\in \mathcal{E}$ and $t'\in \mathcal{E}'$ that are not literals, and for each list of relations $R$  and list of head entities $H$ with $R(H,t)$ in $\mathcal{G}$, and for each list of relations $R'$  and list of head entities $H'$ with $R'(H',t)$ in $\mathcal{G'}$, we use:
\begin{align}
    R(H, t) &\land R'(H', t') \land H \equiv H' \land R \cong R' \nonumber \\
    &\land \textit{fun}(R)  \land \textit{fun}(R,H) \land \textit{fun}(R') \land \textit{fun}(R', H') 
    \quad \stackrel{\text{min}}{\Longrightarrow} \quad t \equiv t'
    \label{eq:generic}
\end{align}
Here, $R(H, t)$ and $R'(H', t')$ are input variables whose value is, by construction,~1. $H \equiv H'$ is an output variable that is implied by the equivalences of each pair of head entities of $H$ and $H'$ through the following rule:
\begin{equation*}
\label{eq:entityset}
   h_1\equiv h'_1\land \ldots \land h_n\equiv h'_n \quad \stackrel{\text{hmean}}{\Longrightarrow} \quad H \equiv H'
   \end{equation*}
The premises of this rule are either themselves output variables of other rules, or   literals, whose similarity is computed upfront (see Section~\ref{sec:algo}). The aggregation function of this rule is the harmonic mean \textit{hmean}, which allows taking into account high evidence (which would be ignored if we used the minimum) while penalizing strong conjunctions with low evidence (unlike the arithmetic mean). 

The term $R \cong R'$ is an output variable that means ``$R$ is similar to $R'$''. It is constrained by the following rule:
\begin{equation*}
\label{eq:relset}r_1\cong r'_1 \land  \ldots \land  r_n\cong r'_n \quad \stackrel{\text{hmean}}{\Longrightarrow}\quad R \cong R'
\end{equation*}
Here, each statement of the form $r\cong r'$ is itself an output variable, which is implied by two rules as follows:
\begin{align*}
\label{eq:sub}r \subseteq r' \quad {\Longrightarrow} \quad r\cong r'\\
r' \subseteq r \quad {\Longrightarrow} \quad r\cong r'
\end{align*}
This means that two relations are considered similar if one is a subrelation of the other (and aggregation is the identity function as there is only one premise).

The functionality terms  are input variables whose values correspond to the functionality or local functionality of the respective lists of relations.  Note that we cannot resort to local functionality only, due to the incompleteness of KGs.
%Since the local functionality is \textit{fun}(\textit{South\-Africa}, \textit{hasCapital}) $= \frac{1}{3}$, the  rule  \eqref{eq:generic} with local functionality avoids this incorrect alignment. 
%Observe that $R(H,t)$ and $R'(H',t')$ are booleans; they depend on the respective knowledge graphs only. The functionality values $F(R), F(R,H), F(R'),$ $ F(R', H')$ belong to $[0,1]$ and also depend  on the respective knowledge graphs only.  
%We introduce global and local functionality to take into account the global graph structure and the local structure surrounding the target entity. Local functionality is quite useful for handling particular cases where the relation or the relation list is not functional. 
For example, people generally have two parents. If one person happens to have only a single parent in one KG, and a single parent in the other, this does not entitle us to match these two parents (as one could be, e.g, the mother and the other the father). The imposition of global functionality avoids such matches.
We cannot use only global functionality either, as a globally functional relationship (such as \emph{has\-Capital}) can be non-functional in some cases (such as \textit{South Africa}).

Finally, $t\equiv t'$ is an output variable that determines to what degree $t$ and $t'$ are matched. We use the \textit{min} aggregation function to make sure all premises of the rule are fulfilled (and a strong premise cannot make up for a weak one). 

%An illustration of how the rules operate on a practical graph is as below: As shown in Figure~\ref{fig:example}, the premises: $R = (\texttt{recordLabel}^{-1}, \texttt{musicalBand}^{-1}, \texttt{releaseDate}^{-1}, \texttt{subsequentWork})$, $R' = (\texttt{P264}^{-1}, \texttt{P175}^{-1}, \texttt{P577}^{-1}, \texttt{P155}^{-1})$, $H = (\texttt{King Records (Japan)}, \texttt{AKB48}, \texttt{"2011-10-26"}, \texttt{Flying Get})$, $H' = (\texttt{Q1328605}, \texttt{Q51118}, \texttt{"2011-10-26"}, \texttt{Q1340881})$, the conclusion $t = \texttt{Kaze was Fuiteiru}$ and $t'=\texttt{Q1340851}$.

%The harmonic mean used in \eqref{eq:entityset} and \eqref{eq:relset} allows us to take into account high evidence (unlike the minimum) while penalizing strongly conjunctions with low evidence (unlike the arithmetic mean). \yp{exist here.}

\paragraph{\textbf{Subrelation alignment}.} To find that one relation $r \in \mathcal{R}$ is a subrelation of another relation $r' \in \mathcal{R'}$, we first build the following rule for all entities $h, t \in \mathcal{E}, h', t' \in \mathcal{E'}$ with $r(h, t), r'(h',t')$:
 $$
    r(h, t) \land r'(h', t') \land h \equiv h' \land t \equiv t'  \quad \stackrel{\text{min}}{\Longrightarrow} \quad r\equiv_{h,t}r'
    $$
As before, $r(h, t), r'(h',t')$ are input variables that evaluate to 1 by construction, while  $h \equiv h'$ and $t \equiv t'$ are output variables of Equation \ref{eq:generic}. Then $r\equiv_{h,t}r'$ is an output variable that says that $r$ and $r'$ coincide on $h$ and $t$. The subrelation alignment can then be written as:
\begin{equation}
\label{eq:subrel}
\bigwedge_{h,t:\, r(h,t)}  r\equiv_{h,t}r' 
\quad 
\stackrel{\alpha\text{-mean}}{\Longrightarrow}
\quad 
r \subseteq r' 
\end{equation}
This rule means that $r$ is a subrelation of $r'$ if all facts of $r$ are also facts of $r'$. To make the degree of $r \subseteq r'$ reflect the proportion of facts of $r$ that are facts of $r'$, we use the arithmetic mean as the aggregation function. However, the arithmetic mean alone is not enough: Under the Open World Assumption, the KG $\mathcal{G}'$ can be incomplete. Thus, some facts $r'(h',t')$ might be missing, which would jeopardise our subrelation alignment. Therefore, we multiply the arithmetic mean by a constant $\alpha$, which we call the \emph{benefit of the doubt} (we show different values of $\alpha$ in our experiments). The aggregation function $\alpha$-mean is then the arithmetic mean multiplied by $\alpha$ and capped to $[0,1]$.%: $\alpha\textit{mean}(x_1,...,x_n):=\textit{min}(\sum_{i=1}^{n} \alpha\times{}x_i/n,1)$.
%If $\alpha = 1$, the score of $r \subseteq r'$ will be exactly as the facts indicates. However, So we set $\alpha \ge 1$ in our rules and the value of $r\subseteq r'$ is clipped to 1. 
%This rule symmetrically applies to inverse subrelation direction $r' \subseteq r$.
% vice versa for inverse matching direction from A to B.

% 
\subsection{Algorithm} % Implementation / Algorithm
\label{sec:algo}

%Our method first computes the similarities of literals between the two KGs and includes training data, if any. The method then 
%We first score literal equivalences for input KGs to initialize entities and subrelation alignments. 
%Starting from these literal pairs, we 
% deduces the alignment of entities (excluding literals) and relations using the above rules, by fixed point iteration. 
 % Fabian: these steps were not yet mentioned, the reader does not expect them
 %The steps of candidate search and maximum assignment are detailed in Section~\ref{sec:details}.
%Since an exhaustive search over all entity and relation lists would be computationally expensive,
%is infeasible for large knowledge graphs, as it constitutes an NP-hard problem. we introduce two strategies: (1) candidate search and (2) bilateral maximum assignment.
%, for improving the efficiency of the algorithm while maintaining its performance. Let us first present the main iterative process along with its proof of convergence, and then a detailed description of the two implementation strategies.
%\vspace{-0.35cm}

%\vspace{-0.72cm}
Our method for KG alignment is shown in Figure~\ref{fig2}. It receives as input two knowledge graphs to be matched, optionally with training data. 

\begin{figure}
\includegraphics[width=\textwidth]{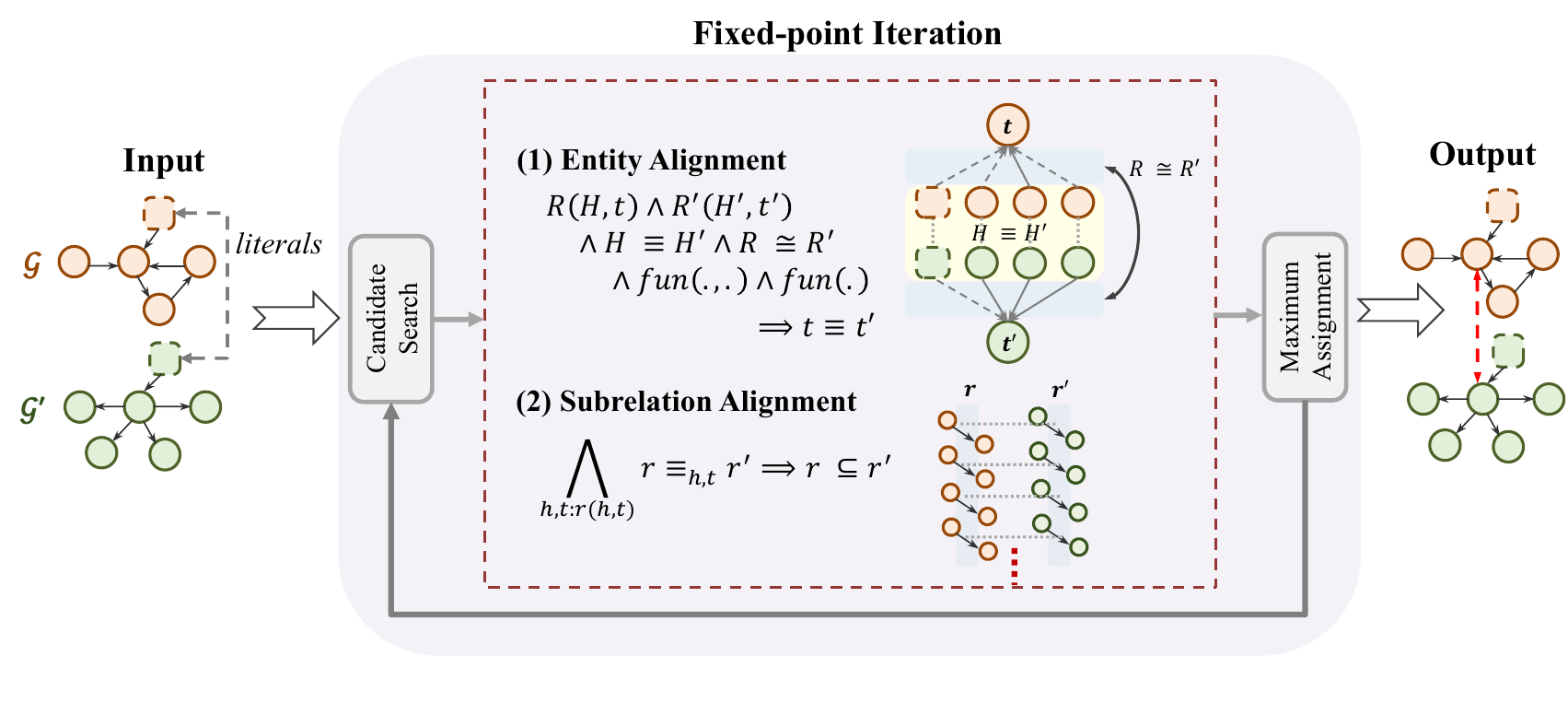}
\caption{Knowledge Graph Alignment with FLORA.} \label{fig2}
\end{figure}

\paragraph{\textbf{Initialization.}} 
We first   compute the similarities between the literals of the two KGs (see details below).  All  literal similarity values are computed only once, belong to $[0,1]$,  and remain fixed throughout the algorithm. 
We integrate training data in the same way: If two entities are matched in the training data, we set their matching score to 1, and treat this match as an input variable that will never be changed.
% After obtaining these literal equivalence pairs, we initialize the entity similarities using flat triples where the facts list is equal to 1. To bootstrap the algorithm, 
To bootstrap the rule of entity alignment of Equation~\ref{eq:generic}, relation similarities are initially set to some small value $\theta_r$. This value is  superseded by the value of the subrelation alignment of Equation~\ref{eq:subrel} in the following iterations. 
%It is recomputed based on $\mathcal{M}_e$ in the second round and no longer $\theta_0$. This initial non-zero alignment enables propagation during the subsequent alignment process. The simplified initialization process is as follows:

% Aftering obtaining these literal equivalence pairs, we init the entity similarities using flat triples where the facts list is equal to 1. The relation alignment is first set to a value $\theta_0$ to bootstrap the algorithm and the reset after entity alignment. This inital alignment are not 0 which prohits the propagation for the following alignment process.
\ignore{ % Fabian: this list does not bring much information over what has already been said
\begin{itemize}
\item Init the matching $e\equiv e' > \theta_1$ for all literals.
    \item Init $e\equiv e'$ using \eqref{eq:generic} for simple facts (entity lists and relation lists of size 1) and $r\cong r'$ set to some value $\theta_0 \in (0,1)$ for all $r\in \mathcal{R}, r'\in \mathcal{R'}$
    \item Reset $r\cong r'$ with the subrelation alignment  \eqref{eq:subrel}.
\end{itemize}
}

\paragraph{\textbf{Fixed point iteration.}}
After initialization, the steps of entity alignment (Equation~\ref{eq:generic}) and subrelation alignment (Equation~\ref{eq:subrel}) are applied alternately.
% \textcolor{red}{The first iteration uses flat triples in place of \fms{in place of what/where?} where the relation list length is set to 1.}
% Yiwen: As all possible relation pairs are initialized to 0.1 at the beginning, the candidate search and facts list can be wrong (as relation is wrong). So in the first iteration, it's better to use flat triples as in Paris to update all possible entity pairs so as to update the subrelations. The following iterations will not have this limit.
If multiple rules imply the same output variable, its score is the maximum of the rule strengths, in line with Algorithm~\ref{alg:fis}. 
%entity or relation pair, we apply $\max$ aggregation, analogous to a disjunction operator in G\"odel fuzzy logic. The $r\cong r'$ is also a disjunction of subrelations in both directions $r\subseteq r' \vee r'\subseteq r$. 
Since all our aggregation functions are continuous and non-decreasing, Theorem~\ref{theo:conv}  applies after the first iteration (when initial values have been set for both entity and relation alignments), proving the convergence of the iterative process to the solution of the corresponding FIS.  

\subsection{Implementation Details} 
\label{sec:details}
%\fms{consider removing if we need space}

% Fabian: I would avoid here any reference to wrong mappings, because this calls into question our entire framework. Better focus on the efficiency aspect.
%Due to its monotonicity, the $\min$ operator in fuzzy logic makes it hard to shrink erroneous alignments to zero and can easily propagate false positives. 
%In addition, 

\paragraph{\textbf{Similarity between literals.}}
The similarity between strings is 
computed by a (small) language model. To reduce the computational cost,  we identify, for each string literal, its most similar string counterparts in the other KG by  cosine similarity, and retain only those above some threshold $\theta_s$.
Dates are  matched exactly, and numbers are matched with a similarity score of 1 if they are approximately equal within a relative error of $10^{-9}$. \added{While more complex matching strategies for dates are possible, experiments in Section~\ref{sec:experiments} show that FLORA performs well even with simple exact matching.}
% \tb{what is a tolerance on a date?}
% \fms{be more concrete}. 

\paragraph{\textbf{Candidate Search.}} Performing a full search over all possible entity and relation combinations would be prohibitively expensive. Therefore, we first perform a candidate search to identify the most structurally similar counterpart for each entity.
%, and then apply rules based on matching sets $\mathcal{M}_e, \mathcal{M}_r$ that are iteratively updated via maximum assignment. 
%These strategies help accelerate convergence and avoid propagating wrong mappings.
% propose two strategies to better integrate our defined rules into practical implementation and to accelerate convergence to a stable value less than 1.
% The search among all lists of entity and relation lists is not feasible for large knowledge graphs, which is an NP hard problem. 
% We use the following heuristics. 
% Assumption; %
%An entity and its matching counterpart usually have similar neighborhood structures. Given the set of matching entities ${\cal M}_e$ and matching relations ${\cal M}_r$ that have been computed so far, % (see the next paragraph for details), 
Specifically, we select for each  entity $t\in \cal{E}$ those entities $t'\in \cal{E}'$ with the maximum number of counterpart triples. These triples are ranked in decreasing order of  matching scores and then considered in this order. 
%\begin{enumerate}\item Retrieve all facts with tail entity $t$, i.e., all pairs $(h,r)$ with $r(h,t)$;\item Retain in this set  all   pairs $(h,r)$ such that $(h,h')\in {\cal M}_e$, $(r,r')\in {\cal M}_r$ and $r'(h',t')$ for some entity $t'$ and relation $r'$;\item In this set, select the entities $t'$ that have the largest number of counterpart pairs $(h, r)$.
%\item For each selected entity $t'$, rank the pairs $(h,r)$ in decreasing order of $h\equiv h'\land r\cong r'$, say $(h_1, r_1),\ldots,(h_N, r_N)$;
 %for each $n=1,\ldots,N$, apply the entity alignment \eqref{eq:generic} to the entity list $H=(h_1,\ldots,h_n)$ and relation list 
%$R=(r_1,\ldots,r_n)$, with corresponding entity list $H'=(h'_1,\ldots,h'_n)$ and relation list 
%$R'=(r'_1,\ldots,r'_n)$, as given by the matching sets ${\cal M}_e$ and ${\cal M}_r$.
%\end{enumerate}

% \subsubsection{Matching sets.} 
\paragraph{\textbf{Maximum Assignment.}} We assume that there are no duplicate entities within each KG, i.e., each entity $e\in {\cal E}$ has at most one counterpart $e'\in {\cal E'}$. Therefore, we store only the highest-scored alignments in the matching set:
%, where $e \equiv e'$ has the highest score mutually between $e$ and $e'$:
% The matching set of entities ${\cal M}_e$ is such that for each entity $e\in {\cal E}$, there is at most one entity 
% $e'\in {\cal E'}$ such that $(e,e')\in {\cal M}_e$. 
% Moreover, we impose that $e\equiv e'$ is positive and
\begin{equation}
    \label{eq:max_assign}
    \mathcal{M}_e \leftarrow 
    \left\{ (e, e') \;\middle|\; 
    e \equiv e' = 
    \max \left\{ 
    \max_{e'' \in \mathcal{E}'}(e \equiv e''),\;
    \max_{e''' \in \mathcal{E}}(e''' \equiv e')
    \right\} \right\}
    % a small error: should be bilateral
    % \mathcal{M}_e \leftarrow \left\{ (e, e') | e\equiv e' = \max_{e''} \{ e\equiv e''\}\right\}
    % \mathcal{M}_e \leftarrow \left\{(e_t, e_t')| e_t \equiv e_t' = 
    % \max \left ( \max_{e \in \mathcal{E}} s(e, e_t'),\,  \max_{e' \in \mathcal{E}'} s(e_t, e') \right ) \right\}.
\end{equation}
We do not impose this one-to-one assumption on relations, and thus store all relation pairs $r\subseteq r'$ with scores greater than 0 in $\mathcal{M}_r$.

% \begin{align*}
% \label{eq:max_assign}
% e' &= \arg\max_{e'\in {\cal E'}} \{e\equiv e'\}\\
% e &= \arg\max_{e\in {\cal E}} \{e\equiv e'\}
% \end{align*}

% \yp{In practice, we don't have this constraints for relations.}
% Similarly, the matching set of relations ${\cal M}_r$ is such that for each relation $r\in {\cal R}$, there is at most one relation 
% $r'\in {\cal R'}$ such that $(r,r')\in {\cal M}_r$. Moreover, we impose that $r\cong r'$ is positive and 
% \begin{align*}
% r' &= \arg\max_{r'\in {\cal R'}} \{r\cong r'\}\\
% r &= \arg\max_{r\in {\cal R}} \{r\cong r'\}
% \end{align*}
% Initially, the matching set of entities ${\cal M}_e$ is restricted to literals and the matching set of relations ${\cal M}_r$ is empty. 

% \yp{we set a bootstrap subrelation score for attributes as 0.1. So $M_r$ is not empty at the beginning, and then correct it iteratively.}

% \subsubsection{Early  stopping.} 
% If no matching increases by more than $\varepsilon$, we stop the iterations.

\paragraph{\textbf{Final alignment.}} The iterations terminate when the total matching score increases by less than some threshold $\varepsilon$ (early stopping).
We keep only those entity matches that exceed a threshold $\theta_e \in [0,1]$. 
%is applied to all entity pairs in ${\cal M}_e$ and relations ${\cal M}_r$. 
The final alignment between two relations $r$ and $r'$ can be a subrelation $r\subseteq r'$, a superrelation $r\supseteq r'$, or \added{an equivalence $r\equiv r'$ if both $r\subseteq r'$ and $r'\subseteq r$ holds.}

% To Be redefined 
% \yp{Yiwen: Rewrite. explain final selection for r = r'}
% Specifically, any pair with a value less than some value $\theta_0 \in(0,1)$ is removed. 
% For evaluation on the property gold standard, we choose the maximally scored property assigned by our approach. To satisfy the 1-to-1 constraints for relations. The score for equivalence refers to page 9.

% Experiments
\section{Experiments} \label{sec:experiments}

\subsection{Experimental Setup}
\paragraph{\textbf{Datasets.}} We evaluate FLORA on five widely-adopted entity alignment datasets, including two monolingual datasets (D-W-15K-V1/V2) from OpenEA~\cite{openea} and three cross-lingual datasets from DBP15K~\cite{jape}. D-W-15K is  based on DBpedia and Wikidata and comes in two versions: Sparse version (D-W-15K-V1) and Dense version (D-W-15K-V2). DBP15K consists of three cross-lingual KG pairs extracted from DBpedia: Chinese and English (DBP\textsubscript{ZH-EN}), Japanese and English (DBP\textsubscript{JA-EN}), French and English (DBP\textsubscript{FR-EN}). For holistic evaluation, we also include two datasets from the OAEI Knowledge Graph Track~\cite{oaeikg}, memoryalpha-stexpanded (Mem-ST) and starwars-swtor (Star-SWT), which have non-trivial matches for relations, classes, and instances.
% \footnote{\textcolor{red}{Detailed statistics of all datasets are available in our GitHub repository}}.
% Note the gold standard for KGsTrack is partial. 

%\vspace{-0.1cm}
\paragraph{\textbf{Parameters.}} As in PARIS~\cite{paris}, the relation similarity of FLORA is initialized to $\theta_r=0.1$. We choose the benefit of doubt as $\alpha=3$ (we show other values in the ablation study). We choose $\theta_s=0.7$ and $\theta_e=0.1$, which maximizes the F1 value in our experiments 
(a smaller value favors recall, a larger value favors precision). Early stopping is set to $\varepsilon=0.01$ (a smaller value improves precision but increases the execution time). 
Regarding the  language models used for string similarity,  we use  LaBSE\footnote{\url{https://huggingface.co/sentence-transformers/LaBSE}} ~\cite{labse} for the multilingual datasets, and PEARL\footnote{\url{https://huggingface.co/Lihuchen/pearl_small}}  \cite{pearl}
for the monolingual datasets \added{(because of its good performance on phrases)}.
%\added{PEARL is optimized for phrase-level similarity, which closely matches the nature of attribute values (literals) of entities.}% It outperforms Sentence-BERT (all-MiniLM-L6-v2) in string matching, entity retrieval, and entity clustering tasks.}

%\vspace{-0.1cm}
\paragraph{\textbf{Evaluation Metrics.}} On D-W-15K-V1/V2 and the OAEI KGs Track, we use standard classification-based metrics (Precision, Recall, F1), while on DBP15K, we follow existing EA baselines by reporting Hit@K and MRR, which measure top-K accuracy and mean reciprocal rank, respectively. To evaluate our model in ranking-based metrics, we directly output all possible entity pairs with scores instead of selecting the top-scoring ones. 

%\vspace{-0.3cm}
\paragraph{\textbf{Competitors.}}
We select more than 30 competitors. %, including 
%methods are selected as baselines for performance comparison of EA and KG alignment tasks,  spanning from 
%supervised and unsupervised, and conventional and state-of-the-art methods.
For the EA task, we include the supervised methods JAPE~\cite{jape}, BootEA~\cite{bootea}, RDGCN~\cite{rdgcn}, OntoEA~\cite{ontoea}, PARIS+~\cite{paris_plus}, RHGN~\cite{rhgn}, GCNAlign~\cite{gcnalign}, TransEdge~\cite{transedge}, FuzzyEA~\cite{fuzzyea}, SelfKG~\cite{selfkg}, AttrGNN~\cite{attrgnn}, TEA~\cite{tea}, SDEA~\cite{sdea}, LLMEA~\cite{llmea}, ChatEA~\cite{chatea} and the unsupervised methods AttrE~\cite{attre}, PARIS~\cite{paris}, PRASE~\cite{prase}, FGWEA~\cite{prase}, NALA~\cite{nala}, LLM4EA~\cite{llm4ea}, CPL-OT~\cite{cpl-ot}, MultiKE~\cite{multike}, ICLEA~\cite{iclea}, BERT-INT~\cite{bert-int}, HLMEA~\cite{hlmea}, and EmbMatch, a simple method using embedding similarity between entity names. For BERT-INT, which uses entity descriptions, we replace the descriptions with entity names for a fair comparison, following~\cite{sdea}.
For holistic KG alignment, we choose classical methods LogMap~\cite{logmap}, AML~\cite{aml}, and the most recent ATMatcher~\cite{atbox}, ALOD2Vec~\cite{alod2vec}, and Wiktionary~\cite{wiktionary}.
We report the results available from the original papers.
% or rerun the experiments from the original code, if available and if not computationally prohibitive.

% Note: For BERT-INT which uses entity descriptions, we replace the descriptions with entity names for a fair comparison following SDEA.

\subsection{Overall Results}  
% Fabian: shortened drastically, I made a backup in folder OLD

\begin{table}[t]
    \centering
    \caption{Performance comparison on the OpenEA datasets (monolingual)}
    \label{tab:openea}
    \begin{adjustbox}{max width=\textwidth}
    \begin{tabular}{clcccccc}
        \toprule
        \multirow{2}{*}{\textbf{Category}} & \multirow{2}{*}{\textbf{Method}} & 
        \multicolumn{3}{c}{\textbf{D-W-15K-V1}} & 
        \multicolumn{3}{c}{\textbf{D-W-15K-V2}} \\
        \cmidrule(lr){3-5} \cmidrule(lr){6-8} 
        & & Precision & Recall & F1 & Precision & Recall & F1 \\
        
        \midrule
        \multirow{6}{*}{\rotatebox{90}{Supervised}} 
        & JAPE & 0.250 & 0.250 & 0.250 & 0.262 & 0.262 & 0.262 \\
        & BootEA & 0.572 & 0.572 & 0.572 & 0.821 & 0.821 & 0.821 \\
        % & RDGCN & 0.515 & 0.515 & 0.515 & 0.623 & 0.623 & 0.623 \\
        & RHGN  & 0.560 & 0.753 & 0.644 & - & - & - \\
        & OntoEA & 0.591 & 0.591 & 0.591 & 0.814 & 0.814 & 0.814 \\
        & PARIS\texttt{+} & \underline{0.965} & \underline{0.746} & \underline{0.841} & \underline{0.976} & \underline{0.903} & \underline{0.938} \\
        
        & \textbf{Ours} & \textbf{0.967} & \textbf{0.856} & \textbf{0.908} & \textbf{0.985} & \textbf{0.966} & \textbf{0.975} \\        
        \midrule
        \multirow{8}{*}{\rotatebox{90}{Unsupervised}}
        & EmbMatch & 0.461 & 0.452 & 0.456 & 0.515 & 0.509 & 0.512 \\
        & AttrE & 0.299 & 0.299 & 0.299 & 0.489 & 0.489 & 0.489 \\
        % & MultiKE (2019) & 41.1 & 41.1 & 41.1 & 49.5 & 49.5 & 49.5 \\ % has different versions
        & PARIS & \underline{0.953} & 0.726 & 0.824 & 0.950 & 0.850 & 0.897 \\
        & PRASE & 0.918 & \underline{0.809} & \underline{0.860} & 0.948 & 0.900 & 0.923 \\ 
        & FGWEA & 0.933 & 0.725 & 0.816 & \underline{0.952}  & 0.903 & \underline{0.927} \\ 
        & NALA & - & - & - & 0.917  & \underline{0.908}  & 0.912 \\
        & LLM4EA & 0.472 & 0.472 & 0.472 & 0.907 & 0.816 & 0.859 \\
        & \textbf{Ours} & \textbf{0.954} & \textbf{0.838} & \textbf{0.893} & \textbf{0.978} & \textbf{0.947} & \textbf{0.962}\\
        \bottomrule
    \end{tabular}
    \end{adjustbox}
\end{table}
\begin{table*}[ht]
\centering
\caption{Performance comparison on the DBP15K dataset across three language pairs (ZH-EN, JA-EN, FR-EN).  Group (I) are structure-only methods, Group (II) uses both structure and entity names, Group (III) jointly uses relational triples, attribute triples, and entity names, and Group (IV) is LLM-based.% for a comprehensive comparison.
}
% Method marked with * use additional information outside the dataset.}
\begin{adjustbox}{max width=\textwidth}
\begin{tabular}{cllccc|ccc|ccc}
\toprule
\multirow{2}{*}{\textbf{Group}} & \multirow{2}{*}{\textbf{Method}} & \multirow{2}{*}{\,\,\,\,\,\,} & \multicolumn{3}{c}{\textbf{DBP\textsubscript{ZH-EN}}} & \multicolumn{3}{c}{\textbf{DBP\textsubscript{JA-EN}}} & \multicolumn{3}{c}{\textbf{DBP\textsubscript{FR-EN}}} \\
 \cmidrule(lr){4-6} \cmidrule(lr){7-9} \cmidrule(lr){10-12} 
& & & Hit@1 & Hit@10 & MRR & Hit@1 & Hit@10 & MRR & Hit@1 & Hit@10 & MRR \\
\midrule
\multirow{3}{*}{I}
% & JAPE      & S & 0.412 & 0.745 & 0.490 & 0.363 & 0.685 & 0.476 & 0.324 & 0.667 & 0.430 \\
& GCNAlign  & S & 0.413 & 0.744 & 0.549 & 0.399 & 0.745 & 0.546 & 0.373 & 0.745 & 0.532 \\
& BootEA    & SS & 0.629 & 0.848 & 0.703 & 0.622 & 0.854 & 0.701 & 0.653 & 0.874 & 0.731 \\
% PARIS+    & S & 0.904 & - & - & 0.874 & - & - & 0.928 & - & - \\
& TransEdge & S & 0.735 & 0.919 & 0.801 & 0.719 & 0.932 & 0.795 & 0.710 & 0.941 & 0.796 \\

\midrule
\multirow{5}{*}{II}
& RDGCN     & S & 0.708 & 0.846 & 0.749 & 0.767 & 0.895 & 0.812 & 0.886 & 0.957 & 0.908 \\
& FuzzyEA   & S & 0.863 & 0.984 & 0.909 & 0.898 & 0.985 & 0.933 & 0.977 & 0.998 & 0.986  \\
& SelfKG    & U & 0.745 & 0.866 & 0.782 & 0.816 & 0.913 & 0.844 & 0.957 & 0.992 & 0.971  \\
& CPL-OT    & U & 0.911 & 0.950 & 0.930 & 0.945 & 0.976 & 0.960 & 0.986 & 0.992 & 0.990  \\
& SEU       & U & 0.900 & 0.965 & 0.924 & 0.956 & \underline{0.991} & 0.969 & 0.988 & \underline{0.999} & 0.992  \\
% SEU       & U  & 0.900 & 0.965 & 0.924 & 0.956 & 0.991 & 0.969 & 0.988 & 0.999 & 0.992 \\
% EASY      & U & 0.898 & 0.979 & 0.930 & 0.943 & 0.990 & 0.960 & 0.980 & 0.998 & 0.990 \\
\midrule
\multirow{8}{*}{III}
& AttrGNN   & S & 0.796 & 0.929 & 0.845 & 0.783 & 0.921 & 0.834 & 0.919 & 0.978 & 0.910 \\
& TEA       & S & 0.941 & 0.983 & 0.960 & 0.941 & 0.979 & 0.959 & 0.979 & 0.997 & 0.991 \\
& SDEA      & S & 0.871 & 0.966 & 0.912 & 0.848 & 0.952 & 0.890 & 0.969 & 0.995 & 0.981 \\
& MultiKE   & U & 0.509 & 0.576 & 0.532 & 0.393 & 0.489 & 0.432 & 0.639 & 0.712 & 0.665  \\
& ICLEA     & U & 0.884 & 0.972 & -     & 0.924 & 0.978 & -     & 0.991 & 0.999 & -     \\
% & BERT-INT  & U & 0.968 & 0.990 & 0.977 & 0.964 & 0.991 & 0.975 & 0.992 & 0.998 & 0.995 \\
& BERT-INT$_{name}$  & U & 0.814 & 0.835 & 0.820 & 0.806 & 0.835 & 0.821 & 0.987 & 0.992 & 0.990 \\
& FGWEA    & U & \textbf{0.976} & \textbf{0.994} & \textbf{0.983} & \textbf{0.978} & \textbf{0.992} & \textbf{0.988} & \textbf{0.997} & \textbf{0.999} & \textbf{0.998} \\

% ZeroEA
% \midrule
& \textbf{Ours} & U & \textbf{0.976} & \underline{0.988} & \underline{0.981} & \underline{0.967} & 0.979 & \underline{0.972} & \underline{0.992} & \underline{0.997} & \underline{0.993} \\
% & \textbf{Ours\texttt{+}} & S & - & - & - & -  & - & - & - & -  & - \\

\midrule
\multirow{3}{*}{IV}
& HLMEA & U & 0.930 & - & 0.934 & 0.938 & - & 0.950 & \textbf{0.986} & - & 0.989 \\
& LLMEA & S & 0.890 & 0.923 & - & 0.911 & 0.946 & - & 0.957 & 0.977 & - \\
& ChatEA & S & \textbf{0.980} & - & \textbf{0.984} & \textbf{0.985} & - & \textbf{0.993} & - & - & - \\
% & LLM-Align & - & 0.983 & - & -- & 0.976 & - & - & 0.995 & - & - \\
\bottomrule
\end{tabular}
\end{adjustbox}
\label{tab:dbp15k_results}
\end{table*}
% \begin{table*}[ht]
% \centering
% % \renewcommand{\arraystretch}{1.2}
% % \setlength{\tabcolsep}{4pt}
% \caption{Performance comparison across methods on OAEI knowledge graphs track datasets. The results are aggregated per method, divided into class, property, and instance.}
% \begin{adjustbox}{max width=\textwidth}
% \begin{tabular}{l|c|ccc|ccc|ccc}
% \toprule
% \multirow{2}{*}{\textbf{Methods}} & \multirow{2}{*}{\texttt{\#cases}} & \multicolumn{3}{c|}{\textbf{class}} & \multicolumn{3}{c|}{\textbf{property}} & \multicolumn{3}{c}{\textbf{instance}} \\
% & & Prec. & F--m. & Rec. & Prec. & F--m. & Rec. & Prec. & F--m. & Rec. \\
% \midrule
% BaselineAltLabel & 5 & 1.00 & 0.71 & 0.59 & 0.99 & 0.76 & 0.66 & 0.89 & 0.84 & 0.80 \\
% BaselineLabel    & 5 & 1.00 & 0.71 & 0.59 & 0.99 & 0.76 & 0.66 & 0.95 & 0.80 & 0.71 \\
% LogMap           & 5 & 0.93 & 0.80 & 0.71 & 0.00 & 0.00 & 0.00 & 0.90 & 0.78 & 0.69 \\
% LogMapLt         & 4 & 0.80 & 0.69 & 0.54 & 0.00 & 0.00 & 0.00 & 0.91 & 0.84 & 0.78 \\
% LSMatch          & 5 & 0.97 & 0.74 & 0.64 & 0.73 & 0.71 & 0.69 & 0.66 & 0.59 & 0.60 \\
% Matcha           & 5 & 0.00 & 0.00 & 0.00 & 0.00 & 0.00 & 0.00 & 0.55 & 0.63 & 0.86 \\
% OLaLa            & 5 & 0.98 & 0.68 & 0.53 & 0.86 & 0.83 & 0.81 & 0.00 & 0.00 & 0.00 \\
% SORBETMtch       & 5 & 0.93 & 0.80 & 0.73 & 0.00 & 0.00 & 0.00 & 0.00 & 0.00 & 0.00 \\
% \midrule
% \textbf{Ours}    & 5 & -- & -- & -- & -- & -- & -- & -- & -- & -- \\
% \bottomrule
% \end{tabular}
% \end{adjustbox}
% \label{tab:ontology_matching}
% \end{table*}

\begin{table*}[ht]
\centering
\caption{Performance comparison on datasets from OAEI}%memoryalpha-stexpanded (Mem-ST), starwars-swtor (Star-SWT) 
 %The overall column shows the performance without dividing into class, property, or instance of the methods.}
\begin{adjustbox}{max width=\textwidth}
\begin{tabular}{p{0.4cm}l|ccc|ccc|ccc|ccc}
\toprule
& \textbf{ } & \multicolumn{3}{c|}{\textbf{Class}} & \multicolumn{3}{c|}{\textbf{Relation}} & \multicolumn{3}{c|}{\textbf{Instance}} & \multicolumn{3}{c}{\textbf{Overall}}\\
\textbf{} & \textbf{} & Prec. & Rec. & F1 & Prec. & Rec. & F1 & Prec. & Rec. & F1 & Prec. & Rec. & F1 \\
\midrule
\multirow{7}{*}{\rotatebox{90}{Mem-ST}} 
& BaselineLabel & 1.00  & 0.46  & 0.63  & 0.97 & 0.68 & 0.80 & 0.98	& 0.84 & 0.91 & 0.98	& 0.83 & 0.90\\
& LogMap & 0.78	& 0.54 & 0.64 & -  & -  & -  & 0.88	& 0.77 & 0.82 & 0.88 & 0.75	& 0.81 \\
& AML & 1.00 & 0.69 & 0.82 & 0.89 & 0.80 & 0.85	& 0.93	& 0.93	& 0.93  & 0.93	& 0.92 & 0.93	\\
& ATMatcher & 1.00	& 0.77 & 0.87	& 0.95  & 0.95 & 0.95  & 0.96	& 0.92  & 0.94	& 0.96	& 0.92 & 0.94	\\
% & Fine-TOM & 1.00 & 0.62 & 0.76	& .236  & .449  & .339  & .158  & .277  & .219  \\
& ALOD2Vec & 1.00 & 0.62 & 0.76 & 0.87	& 0.95 & 0.91 & 0.92	& 0.93	& 0.93 & 0.92	& 0.93 & 0.92 \\
& Wiktionary & 1.00 & 0.62 & 0.76 & 0.87 & 0.95 & 0.91	& 0.92	& 0.93	& 0.93 & 0.92	& 0.93 & 0.92 \\
% & \textbf{Ours} & \textbf{1.00} & \textbf{0.85} & \textbf{0.92} & 0.98 & 0.95 & 0.96 & 0.96  &  0.95  & 0.95  & 0.96 & 0.95 & 0.95\\
& \textbf{Ours} & \textbf{1.00} & \textbf{0.85} & \textbf{0.92} & \textbf{0.98} & \textbf{0.98} & \textbf{0.98} & \textbf{0.96}  &  \textbf{0.95}  & \textbf{0.95}  & \textbf{0.96} & \textbf{0.95} & \textbf{0.95}\\

\midrule
\multirow{7}{*}{\rotatebox{90}{Star-SWT}} 
& BaselineLabel & 1.00	& 0.80 & 0.89 & \textbf{1.00} &	0.89 & 0.94  & 0.95	& 0.84 & 0.89	& 0.95	& 0.84 & 0.89	\\
& LogMap & 	1.00 &	0.73 &  0.85 & -  & -  & -  & 0.94	& 0.78 & 0.86 & 0.94 & 0.75
& 0.84 \\
& AML & 1.00	& 0.87 & 0.93 & 0.98 & 0.71  & 0.82	& 0.93	& 0.90 & 0.92 & 0.93	& 0.90 & 0.91	\\
& ATMatcher & 1.00	& 0.87 & 0.93	& \textbf{1.00}	& 0.98 & \textbf{0.99} & 0.94 & 0.91 & 0.92  & 0.95	& 0.91 & 0.93	\\
& ALOD2Vec & 1.00	& 0.87 & 0.93	& 0.98	& 0.98	& 0.98  & 0.92	& 0.91 & 0.92	& 0.93	& 0.92  & 0.92	\\
& Wiktionary & 1.00	& 0.87 & 0.93	& 0.98	& 0.98	& 0.98  & 0.92	& 0.91  & 0.92 & 0.93	& 0.92	& 0.92\\
% & \textbf{Ours} & 1.00  & 1.00  & 1.00 & 0.98  & 0.98  & 0.98  & 0.98  & 0.96  & 0.97  & 0.98 & 0.96 & 0.97\\
& \textbf{Ours} & \textbf{1.00}  & \textbf{1.00}  & \textbf{1.00} & 0.98  & \textbf{0.98}  & 0.98  & \textbf{0.98}  &  \textbf{0.96}  & \textbf{0.97}  & \textbf{0.98} & \textbf{0.96} & \textbf{0.97}\\

\bottomrule
\end{tabular}
\end{adjustbox}
\label{tab:ontology_matching}
\end{table*}

Table~\ref{tab:openea} shows the results  on monolingual EA in both unsupervised and supervised settings.  FLORA achieves considerable improvements over all baselines,  with and without training data. Table~\ref{tab:dbp15k_results} shows the results for multi-lingual EA. FLORA consistently performs the best or the second-best across all datasets. 
FLORA does not beat FGWEA, which is designed for multilingual datasets and exploits the name bias present in such benchmarks~\cite{attrgnn}. Therefore, FGWEA performs worse than FLORA when entity names are incomplete, as observed in the monolingual EA datasets. FGWEA further assumes that each entity of one KG has one match in the other KG, which is not representative of  real-world scenarios. 
Regarding  ChatEA, it is supervised and additionally uses the background knowledge of LLMs to generate entity descriptions. In contrast, FLORA is unsupervised, relies solely on the given KGs data, and needs only a small language model for string similarity. 
Overall, the difference of FLORA to the best performer is less than 1 percentage point -- a price we pay for a fully transparent unsupervised algorithm that works across different scenarios. 

Table~\ref{tab:ontology_matching} further shows the performance of FLORA on the KG alignment task. 
% \added{To satisfy the one-to-one constraint for relation equivalence, we choose the maximally scored assignment for each relation.}
%\textcolor{red}{For evaluation on the property gold standard, we choose the maximally scored property assigned by our approach.}
FLORA achieves the highest overall F1 score with an average of 96\%,  outperforming competitors for the alignment of classes, relations, and instances. 
It  falls within 2 percentage points of the best model for relation matching precision on the Star-SWT dataset. 
This slight drop is because FLORA finds structural matches (\textit{apprentice} matches \textit{padawan}, which also means apprentice) rather than the intended semantic matches (\textit{apprentice} matches \textit{apprentices}).
%being structurally overshadowed by other candidates (e.g., \textit{padawan}), which pushes the gold-standard match (\textit{apprentices},\textit{apprentice}) to the second place.

% It falls within a 2 percentage point range of the best performer. 

% \subsection{Ablation Studies} \fms{I might have shortened this too much :-) Have to move back in the material form the backup :-) }
% For space reasons, we defer to the appendix (available in the supplementary material) our studies of the impact of the benefit-of-doubt parameter $\alpha$, of using string identity instead of a language model for string similarity, and of our candidate search optimization. 

% Yiwen: I put it back. I think these analyses are quite important.
\subsection{Analysis}
\begin{table*}[t]
\centering
\caption{Ablation studies of FLORA on various datasets 
%T/O denotes that the algorithm cannot finish within the allowed time budget (24 hours).
}
\begin{adjustbox}{max width=\textwidth}
\begin{tabular}{lccc|ccc|ccc|ccc}
\toprule
\multirow{2}{*}{\textbf{Method}} & \multicolumn{3}{c|}{\textbf{D-W-15K-V1}} & \multicolumn{3}{c|}{\textbf{D-Y-15K-V2}} & \multicolumn{3}{c|}{\textbf{DBP\textsubscript{FR-EN}}} & \multicolumn{3}{c}{\textbf{DBP\textsubscript{ZH-EN}}} \\
& Prec. & Rec. & F1 & Prec. & Rec. & F1 & Hit@1 & Hit@10 & MRR & Hit@1 & Hit@10 & MRR \\
\midrule
\textbf{FLORA}     & \textbf{0.954}     & \textbf{0.838}     & \textbf{0.893}  & \textbf{0.978}     & \textbf{0.947}     & \textbf{0.962}  & \textbf{0.992} & \textbf{0.997} & \textbf{0.993}  & \textbf{0.976} & \textbf{0.988} & \textbf{0.981} \\
\quad - w/o literal similarity            & 0.939     & 0.797    & 0.862  & 0.962     & 0.925     & 0.943   & 0.975     & 0.983     & 0.977     & 0.947& 0.965& 0.951\\
\quad - w/o list of relations         & 0.894     & 0.604    & 0.721  & 0.918     & 0.628    & 0.756  & 0.933 & 0.964& 0.944& 0.700     & 0.852     & 0.757  \\
% \quad - with \textit{min} aggregation  & 0.934    & 0.804    & 0.864    & 0.965     & 0.933    & 0.949 &  0.984   & 0.991 & 0.986  & 0.967    & 0.981    & 0.970 \\
% Yiwen: New
\quad - with \textit{min} aggregation  & 0.564  & 0.524  & 0.543  & 0.659 & 0.626 & 0.641 & 0.894  & 0.949  & 0.912 & 0.547  & 0.778  & 0.632 \\
%\quad - w/o Candidate search     & \textit{T/O}    & \textit{T/O}    & \textit{T/O}  & \textit{T/O}   & \textit{T/O}     & \textit{T/O} & \textit{T/O}     & \textit{T/O}    & \textit{T/O}  & \textit{T/O}    & \textit{T/O}    & \textit{T/O}  \\
\quad - with $\alpha=100$ & 0.940 & 0.807  & 0.868 & 0.956  & 0.928 & 0.937 & 0.974 & 0.985  & 0.976 & 0.871 & 0.935 & 0.898 \\
\bottomrule
\end{tabular}
\end{adjustbox}
\label{tab:ablation}
\end{table*}
\begin{figure}[t]
    \centering
    % \begin{subfigure}[b]{0.4\textwidth}
    %     \includegraphics[width=\textwidth]{figures/D-W-15k-V1.pdf}
    %     \caption{D-W-15k-V1}
    % \end{subfigure}
    \begin{subfigure}[b]{0.328\textwidth}
        \includegraphics[width=\textwidth]{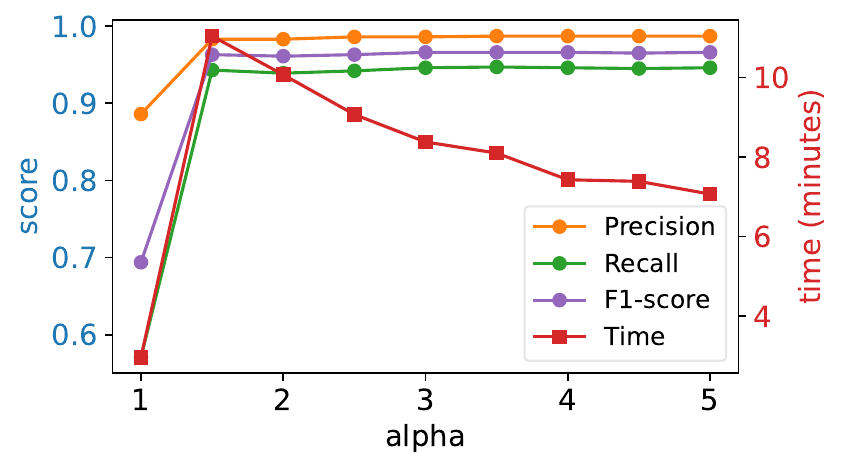}
        \caption{D-W-15K-V2}
    \end{subfigure}
    \begin{subfigure}[b]{0.328\textwidth}
        \includegraphics[width=\textwidth]{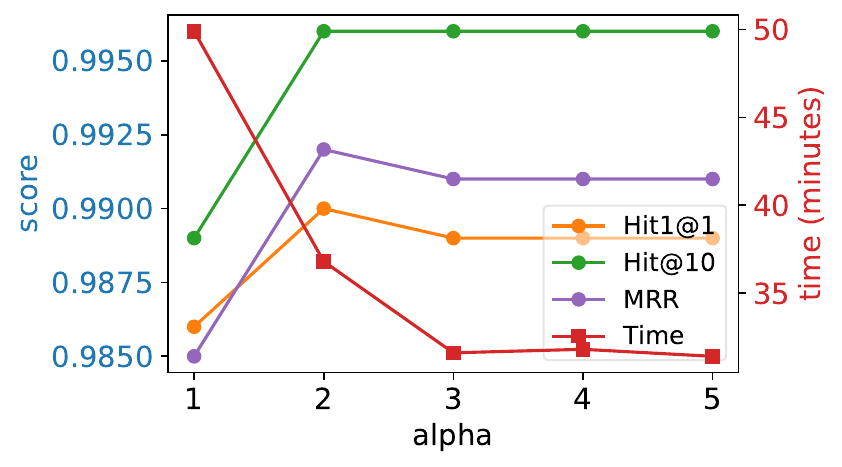}
        \caption{DBP\textsubscript{FR-EN}}
    \end{subfigure}
    \begin{subfigure}[b]{0.328\textwidth}
        \includegraphics[width=\textwidth]{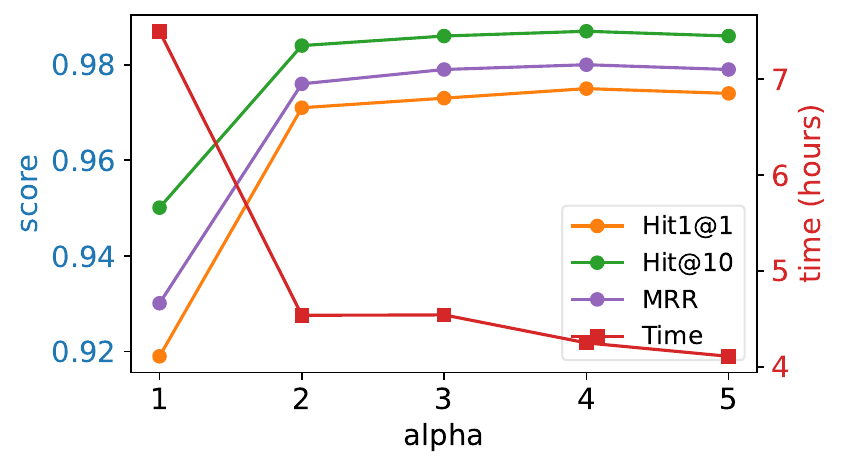}
        \caption{DBP\textsubscript{ZH-EN}}
    \end{subfigure}
    \caption{Sensitivity to the benefit of doubt $\alpha$}\vspace{-5mm}
    \label{fig:dbp15k_subfig}
\end{figure}

%\fms{Maybe shorten as follows:}
\paragraph{\textbf{Ablation study}.}
 We first experimented with different values of the benefit-of-doubt parameter $\alpha$ (see Figure \ref{fig:dbp15k_subfig}). 
We find that setting $\alpha = 1$ is ineffective and leads to a significant reduction in performance, thus proving the necessity of $\alpha>1$. 
Beyond $\alpha=2$, the performance does not change much while computation time first decreases and then remains stable. However, a larger $\alpha$ does not necessarily yield better results, as the performance degrades once $\alpha$ goes beyond 100 (see Table~\ref{tab:ablation}). In practice, we set $\alpha=3$ to balance effectiveness and efficiency.
%We thus conclude that any value of $\alpha$ between 2 ad We also find that while the performance continues to improve slightly beyond $\alpha=2$, the computational time decreases significantly with larger values of $\alpha$. 
% \yp{do we explain that $\alpha$ cannot be too large?}
We also replaced the language model used for string similarity by simple string identity (see Table \ref{tab:ablation}). 
This leads to an overall performance drop, especially on DBP\textsubscript{ZH-EN}, where entities have very different literal values across languages. But even with identity literal initialization, the score drops by 3 percentage points at most, and FLORA still outperforms strong baselines such as RDGCN and SelfKG, underscoring the strength of our structural matching.
Next, we restrict the list of relations in entity alignment (Equation~\ref{eq:generic}) to length 1, thereby removing the structural signal from our rules. In this setting, our system can propagate only alignment scores through flat triples, similar to  PARIS~\cite{paris}. This significantly degrades performance, demonstrating the importance of structural information in the absence of functional relations. 
%Fabian: removed for spacing reasons...
Finally, using the \textit{min} aggregation function  for all inference rules significantly degrades the performance across all datasets.
\paragraph{\textbf{Explainability.}} A salient feature of FLORA is that its results are interpretable by humans. For any pair of aligned entities, the corresponding rule with the highest firing strength, as given by Equation~\ref{eq:generic}, explains the alignment. In the dataset D-W-15K-V1 for instance, the  entities $t$=\href{https://dbpedia.org/page/Lady_Gaga}{\textit{Lady Gaga}}  of DBPedia and $t'$=\href{https://www.wikidata.org/wiki/Q19848}{Q19848} of Wikidata are aligned with a value of $0.984$. The rule leading to this alignment states  that Lady Gaga was born on \textit{1986-03-28} and featured in the song \textit{3-Way}; it is the alignment of the relations \textit{musicalArtist} $\cong$ P175
and 
\textit{birthDate}$^{-1}\cong$ P569$^{-1}$, the functionality of the corresponding relation set, and  the alignment of the birth dates and of the entities \textit{3-Way} and Q659417, that lead to the final alignment of \textit{Lady Gaga} and Q19848. 
%Observe that the alignment of  \textit{3-Way (The Golden Rule)} and Q659417 can itself be explained by a similar rule, with another set of relations.

% the dates (literals) and of the entities \textit{3-Way (The Golden Rule)}$\equiv$Q659417=0.968, and the alignment of 
% % (which can itself be explained) 
% that lead to the final alignment of \href{https://dbpedia.org/page/Lady_Gaga}{\textit{Lady Gaga}} and \href{https://www.wikidata.org/wiki/Q19848}{Q19848}.

%\paragraph{\textbf{Limitations.}}\fms{consider removing...}\tb{I would keep it; some reviewers expect that paragraph.} Despite its strong performance, FLORA has some limitations. First, it does not support asymmetric matching between entities. 
%For instance, in the DBP\textsubscript{FR-EN} dataset, the entity \textit{Max\_Born} has two facts: \textit{<dbr-fr:Max\_Born, dbp-fr:lieuDeDécès, dbr-fr:Allemagne>} and \textit{<dbr:Max\_Born, dbp:deathPlace, dbr:} \textit{ West\_Germany>} in two KGs respectively. During alignment, \textit{dbr-fr:Allemagne} is matched to \textit{dbr:Germany} instead of \textit{dbr:West\_Germany}, causing a loss of this fine-grained information used for matching \textit{Max\_Born} through its \textit{deathplace} relation. 
%
%This issue is  evident for aligning classes, where \textit{subClassOf} relations introduce granularity mismatches. FLORA currently neither handles transitive closure of subclasses nor the acyclic taxonomy structure. Second, FLORA relies on attribute triples and cannot operate in settings where only structure information is available.

% Yiwen: Below the paragraph is new!
\paragraph{\added{\textbf{Scalability and Efficiency.}}}

\begin{table}[t]
    \centering
    \caption{Performance comparison on the OpenEA 100K datasets (monolingual)}
    \label{tab:larger}
    \begin{adjustbox}{max width=\textwidth}
    \begin{tabular}{lccccccc}
        \toprule
        \multirow{2}{*}{\textbf{Method}} & \multirow{2}{*}{\textbf{Supervised}} & 
        \multicolumn{3}{c}{\textbf{D-W-100K-V1}} & 
        \multicolumn{3}{c}{\textbf{D-W-100K-V2}} \\
        \cmidrule(lr){3-5} \cmidrule(lr){6-8} 
        & & Precision & Recall & F1 & Precision & Recall & F1 \\
        
        \midrule
        % \multirow{6}{*}{\rotatebox{90}{Supervised}} 
        RDGCN & Yes & 0.362  & 0.362 & 0.362 & 0.421 & 0.421 & 0.421 \\
        BootEA & Yes & 0.516 & 0.516 & 0.516 & 0.766 & 0.766 & 0.766 \\ 
        \midrule
        PARIS & No & 0.934 & 0.644 & 0.762 & 0.933 & 0.796 & 0.859 \\
        PRASE & No & - & - & - & 0.927 & 0.855 & 0.890 \\
        % LightEA-I & No & 0.642 & 0.642 & 0.642 & 0.926 & 0.926 & 0.926 \\
        NALA & No & 0.732 & 0.732 & 0.732 & 0.919 & 0.919 & 0.919\\
        \textbf{FLORA} & No & \textbf{0.937} & \textbf{0.773} & \textbf{0.847} & \textbf{0.970} & \textbf{0.929} & \textbf{0.949} \\
        \bottomrule
    \end{tabular}
    \end{adjustbox}
\end{table}

Finally, we evaluate the scalability and efficiency of our approach. For scalability, we use a larger-scale entity alignment dataset from OpenEA, D-W-100K (with 100k gold alignments). As shown in Table~\ref{tab:larger}, FLORA outperforms both advanced unsupervised and supervised baselines. 

For efficiency, \added{we measured the run time of FLORA on datasets of different sizes from our previous experiments, using a Linux server with an AMD EPYC 9374F CPU and 519 GB of RAM.}
% datasets on a Linux server with an AMD EPYC 9374F CPU and 519 GB of RAM. 
FLORA takes approximately 7 min on D-W-15K(V1/V2) and 48 min on the larger dataset D-W-100K-V1. This is comparable to most state-of-the-art methods~\cite{openea}. However, the other methods are executed in GPUs, while our approach needs only CPUs. % due to its logical design, yet remains efficient.
% \yp{No need to say why?}
% And add a sentence on CPU vs GPU (e.g., what make other methods suitable for GPU and not ours?)
% (XXX)
On the DBP15K datasets, FLORA is much slower than on the D-W-15K datasets, with around 2 hours on average. %) is higher than on the OpenEA D-W-15K datasets.}
% \fms{than what?} . 
This is likely due to the amount of auxiliary information (i.e., attribute values), which is about five times larger in the DBP15K datasets than in the D-W-15K datasets. %\added{This typically increases runtime but enhances performance.} % Fabian: I think we can do without an explanation -- if it's 5 times larger, then the runtime is slower...
% \footnote{Detailed statistics of datasets are available in our GitHub repository.} 
%\added{The OAEI platform provides the runtimes of all compared baselines of the KG Track~\cite{oaei2021}, which are executed 
%on a virtual machine with 32GB of RAM and 16 vCPUs (2.4 GHz).} 
On the OAEI KG Track datasets, FLORA completes the alignment process in about 1h45min. This is much slower than the runtime of lexical-based methods (e.g, 4 min on average for ALOD2Vec) that has been reported by the OAEI platform on a virtual machine with 32GB of RAM and 16 vCPUs (2.4 GHz)~\cite{oaei2021}. This is because FLORA relies more on structural reasoning, which, in turn, leads to better performance,  allows for explainability, and can find not only relation equivalence but also relation subsumption.
\section{Conclusion and Future Work} \label{sec:conclusion}

In this paper, we have introduced FLORA, an iterative method based on Fuzzy Logic that aligns entities and relations across knowledge graphs. FLORA provably converges to the solution of the corresponding Fuzzy Inference System. It can find  subrelation matches, it can deal with dangling entities, it does not need training data, and it provides interpretable results. Comparative experiments  on major EA and KG alignment benchmarks  show that FLORA outperforms all competitors on nearly all datasets. Future work could extend FLORA to more complex taxonomy structures, for example, by exploring cross-KGs subclasses, computing deductive closures, and taking into account the graph structure.

\noindent\textbf{\textit{Supplementary Material Statement.}} The dataset, source code, supplementary material, and hyperparameters are available at our GitHub repository \url{https://github.com/dig-team/FLORA}.

\clearpage

\bibliographystyle{splncs04}
\bibliography{mybibliography}




% LNCS Homepage, \url{http://www.springer.com/lncs}, last accessed 2023/10/25
% \end{thebibliography}

% \appendix
% \include{contents/appendix}
\end{document}